\documentclass[conference]{IEEEtran}
\IEEEoverridecommandlockouts
\usepackage{cite}
\usepackage{amsmath,amssymb,amsfonts}
\usepackage{mathtools}
\usepackage{amsthm}
\usepackage{graphicx}
\usepackage{textcomp}
\usepackage{microtype}
\usepackage{float}
\usepackage{subfigure}
\usepackage{booktabs} 
\usepackage{epstopdf}
\usepackage{stfloats}
\usepackage{multirow}
\usepackage{makecell}
\usepackage{colortbl}
\usepackage[dvipsnames]{xcolor}
\usepackage{algorithm}
\usepackage{algorithmic}

\usepackage{hyperref}

\def\BibTeX{{\rm B\kern-.05em{\sc i\kern-.025em b}\kern-.08em
    T\kern-.1667em\lower.7ex\hbox{E}\kern-.125emX}}
    
\newcommand{\ie}{\emph{i.e.}}
\newcommand{\eg}{\emph{e.g.}}
\newcommand{\etal}{{\em et al. }}
\newtheorem{assumption}{Assumption}

\newtheorem{lemma}{Lemma}
\newtheorem{theorem}{Theorem}

\newtheorem{claim}{Claim}
\theoremstyle{plain}

\newcommand{\para}[1]{\noindent {\bf #1}}

\newcommand{\ryan}[1]{\textcolor{purple}{[Ryan: #1]}}

\newcommand{\sys}{WPM}
\newcommand{\accLossScale}{0.29}
\newcommand{\convergenceSpeedScale}{0.32}
\newcommand{\picRatio}{.22}
\newcommand{\tableCellWidth}{1.6cm}
\newcommand{\conSpeedCellWidth}{1.89cm}
\newcommand{\disTableCellWidth}{1cm}

\DeclareMathOperator*{\argmin}{arg\,min}

\begin{document}

\title{Achieving Efficient Distributed Machine Learning \\ Using a Novel Non-Linear Class of Aggregation Functions}


\author{\IEEEauthorblockN{1\textsuperscript{st} Haizhou Du}
\IEEEauthorblockA{\textit{Shanghai University of Electric Power} \\
}
\and
\IEEEauthorblockN{2\textsuperscript{nd} Ryan Yang}
\IEEEauthorblockA{\textit{Choate Rosemary Hall} \\
}
\and
\IEEEauthorblockN{3\textsuperscript{rd} Yijian Chen}
\IEEEauthorblockA{\textit{Shanghai University of Electric Power} \\
}
\and
\IEEEauthorblockN{4\textsuperscript{th} Qiao Xiang}
\IEEEauthorblockA{\textit{Xiamen University} \\
}
\and
\IEEEauthorblockN{5\textsuperscript{th} WIbisono Andre}
\IEEEauthorblockA{\textit{Yale University.} \\
}
\and
\IEEEauthorblockN{6\textsuperscript{th} Wei Huang}
\IEEEauthorblockA{\textit{Shanghai University of Electric Power} \\
}
}

\maketitle

\begin{abstract}
Distributed machine learning (DML) over time-varying networks can be an enabler for emerging decentralized ML applications such as autonomous driving and drone fleeting. However, the commonly used weighted arithmetic mean model aggregation function in existing DML systems can result in high model loss, low model accuracy, and slow convergence speed over time-varying networks. To address this issue, in this paper, we propose a novel non-linear class of model aggregation functions to achieve efficient DML over time-varying networks. Instead of taking a linear aggregation of neighboring models as most existing studies do, our mechanism
    uses a nonlinear aggregation,
    a weighted power-$p$  mean (WPM) where $p$ is a positive integer, as the aggregation function of local models from neighbors. The subsequent optimizing steps are taken using mirror descent defined by a Bregman divergence that maintains convergence to optimality. In this paper, we analyze properties of the WPM and rigorously prove convergence properties of our aggregation mechanism. Additionally, through extensive experiments, we show that when $p>1$, our design significantly improves the convergence speed of the model and the scalability of DML under time-varying networks compared with arithmetic mean aggregation functions, with  little additional computation overhead. 
\end{abstract}

\begin{IEEEkeywords}
Distributed Machine Learning, Non-Linear Aggregation Function, Time-varying, Weighted Power Mean
\end{IEEEkeywords}

\section{Introduction}
\label{sec:intro}

Edge devices (\eg, cell phones, tablets and cameras) generate massive amounts of data.
A survey from Cisco~\cite{cisco2020} estimates that by 2021, almost 850 ZB (ZettaBytes) of data will be generated by edge devices. Because of resource and privacy constraints, gathering all data for centralized processing can be impractical.  Distributed machine learning (DML) has emerged as a powerful solution that is adept at analyzing and processing such data to support data-driven tasks (\eg, autonomous driving, virtual reality and smart grid). For example, Tesla motors use GPUs as computer vision accelerators for its full self-driving in real-time scenario \cite{tesla2020}. By having devices perform training tasks locally and then aggregating the local models, DML let devices collaboratively build a global model, and preserves the users' privacy for edge AI~\cite{koloskova2019decentralized}.

To avoid the scalability bottleneck and single point of failure issues brought by using a parameter server~\cite{smola2010architecture}, recent DML systems propose to use decentralized architectures (\eg, All-Reduce~\cite{allreduce2010}, Ring-Reduce~\cite{lee2020tornadoaggregate} and Gossip~\cite{hegedHus2019gossip}), where devices share their local models with neighbors, and aggregate received local models. For example,   
Xiao \etal propose a linear iteration to yield distributed averaging consensus over a network based on the Laplacian matrix of the associated graph \cite{1272421}. Decentralized Stochastic Gradient Descent (DSGD)  methods propose to trade the exactness of the averaging provided by All-Reduce~\cite{candecentralizedoutperform,hegedHus2019gossip,nedic2020distributed} for better scalability. 
Heged \etal propose a gossip-based server architecture to solve decentralized learning in mobile computing ~\cite{hegedHus2019gossip}.

Despite substantial efforts from existing literature, they ignore the \textit{topology dynamicity}, a fundamental challenge for applying DML in edge computing. Specifically, edge devices usually exchange local models through wireless communication, which is highly dynamic. As a result, two devices may be able to exchange local models in one iteration, but lose the connection in the next iteration. Moreover, doing so drastically slows down the convergence speed because the communication is bottlenecked by the limited wide-area or mobile edge network bandwidth. Specifically, when the training process iterates  over time-varying  networks~\cite{kovalev2021linearly, pmlrv139kovalev21a}, 
this often leads to a severe drop in the training and test performance (\ie, generalization gap), even after hyper-parameter fine-tuning. 
 Although some papers have proposed a linear class of aggregation functions for DML over time-varying networks and proven their convergence~\cite{nedic2020distributed,neglia2020decentralized,pmlr-v119-eshraghi20a,pmlrv139kovalev21a}, the convergence speed still can not match the requirements of emerging applications.

To address this challenge, we propose a novel non-linear model aggregation mechanism to achieve efficient DML over time-varying networks. We make a key observation that existing DML systems exclusively use the linear weighted-mean of neighboring local models as their aggregation functions, with different ways of assigning the appropriate weights for averaging local models.
In contrast, instead of fine-tuning the weighting of local models during aggregation or designing  delicate and complex
architecture, our mechanism takes a weighted power-$p$ mean (WPM) of local models from neighbors when $p$ is odd, and a modified version for even $p$. Then, the local model takes a mirror descent step, which is a generalization of normal gradient descent with the Bregman divergence as distance-measuring function.  The core insight of the WPM mechanism is that 
the variance among the local models of  devices estimates is lower when $p$ is higher. When the variance is lower, DML more efficiently simulates the centralized gradient descent method. And the effect of this insight will be further amplified in time-varying, sparse networks, which are common in edge computing. Results of extensive experiments show that by using WPM as the aggregation function, the convergence speed of DML model over time-varying networks accelerates up to 62\%, and the accuracy of DML model is also higher than that achieved by using the commonly used linear arithmetic mean aggregation in terms of the scalabilty metric. 

The \textbf{main contributions} of this paper are as follows:

\begin{itemize}
\setlength\itemsep{0em}
    \item We study the important issue of DML over time-varying networks, which can help boost the performance of decentralized edge applications such as autonomous driving, drone fleeting, and mobile image classification. 
	\item We propose a novel model aggregation mechanism that improves the convergence speed and scalability by taking a weighted power mean of local models from neighbors based on approximate mirror descent. To the best of our knowledge, we are the first to propose such a non-linear model aggregation mechanism. 
	
	\item We conduct extensive experiments to evaluate the performance of our mechanism. Experimental results show that WPM mechanism accelerates the convergence speed by up to 62\%, improves model accuracy with scalability metric over time-varying networks
	than the state-of-the-arts.
\end{itemize}

The rest of this paper is organized as follows. Section \ref{sec:design} describes mechanism design and WPM definition. In section \ref{sec:analysis}, we discuss properties/observations relevant to WPM. Section \ref{sec:evaluations} gives a comprehensive evaluation of our work. Section \ref{sec:related} describes the related works and we conclude the paper in Section \ref{sec:con}.

\section{Related Work}
\label{sec:related}

\para{DML for Edge Computing.} 
Many distributed machine learning models and its variants have been proposed to process huge amount of data locally in edge computing \cite{verbraeken2020survey}. 
For example, 
Donahue \etal~\cite{Donahue_Kleinberg_2021} propose a  federated learning framework of coalitional game theory to eliminate the bias of a global aggregation model. Lee \etal ~\cite{lee2020flexreduce} propose a FlexReduce method on asymmetric network topology to speed up convergence. Wang \etal~\cite{8664630} propose a control algorithm to  tradeoff between local update and global parameter aggregation under a given resource budget. Huang \etal~\cite{citekey} propose a distributed deep learning-based offloading algorithm for MEC networks. Ghosh \etal ~\cite{Efficient2020} propose a framework of iterative federated clustering learning. 
Li \etal \cite{parameterserver2014} propose a parameter-server-based architecture for DML. 
Lin \etal propose a novel momentum-based method to mitigate this decentralized training difficulty ~\cite{pmlr-v139-lin21c}.
However, they assume a fixed topology and are not suitable for edge computing under network dynamics. 


\para{DML with Topology Dynamicity.}
Kovalev \etal propose the ADOM method for decentralized optimization over time-varying networks with projected Nesterov gradient descent~\cite{pmlrv139kovalev21a}. Koloskova \etal introduce a framework covering local SGD updates and synchronous and pairwise gossip updates on adaptive network topology~\cite{pmlr-v119-koloskova20a}. Eshraghi \etal propose a DABMD algorithm for decentralized learning by varying mini-batch sizes across time-varying topology~\cite{pmlr-v119-eshraghi20a}. Pu \etal consider new distributed gradient-based methods with an estimate of the optimal decision variable and an estimate of the gradient for the average of the agents' objective functions~\cite{pu2020push}. Nedic \etal \cite{nedic2009distributed,nedic2020distributed} tackle the DML with topology dynamicity from the consensus perspective. They propose a model aggregation method for agents in a network with a time-varying topology to collaboratively solve a convex objective function, with a convergence guarantee. However, our preliminary results show that their methods suffers from lower accuracy, higher loss and slower convergence speed under time-varying networks compared to under a fixed topology. 
Additionally, their methods use linear averages for aggregation, which is a special case of our framework.

\section{Design}
\label{sec:design}


\subsection{Problem Formulation}
We consider the scenario of DML on a total of $m$ devices over time-varying networks, without the existence of a centralized parameter server. Such time-varying networks are ubiquitous in many complex systems and practical applications (\eg, sensor networks, next generation federated learning systems), where the communication pattern among pairs of mobile devices will be influenced by their physical proximity, which naturally changes over time~\cite{zadeh1961time,kolar2010estimating,pmlrv139kovalev21a}.  The connection between any two devices $i$ and $j$ is dynamic but symmetric. In other words, in the same epoch, if $i$ can send information to $j$, $j$ must be able to send information to $i$ as well; if $i$ cannot send information to $j$, then neither can $j$ to $i$. Each device $i$ has its own dataset and loss function $F_i$. 
Denote the model parameters as a vector $\mathbf{w} \in \mathbb{R}^n$. Devices aim to collaboratively minimize the following objective function:
\begin{equation}
\label{eq:problem}
 \text{minimize } f(\mathbf{w}) = \sum_{i=1}^m F_i(\mathbf{w}), \text{ subject to } \mathbf{w}\in \mathbb{R}^n,
\end{equation}
without sharing any data among devices. 

\subsection{Mechanism Overview}

After presenting the problem formulation, we propose our mechanism to solve it efficiently under dynamic topologies. Algorithm \ref{alg:new} describes the overall training process of our mechanism on each device.

Specifically, in each iteration $t$, each device $i$ uses its local dataset to compute its local gradient $d_i(t)$. It then checks its wireless connections with other devices, and sends to each connected device $j$ the following information: (1) $\mathbf{w}_i(t)$, the local model parameter of $i$ at the beginning of iteration $t$; and (2) $e_{i, j}(t) = \frac{1}{N_i(t)+1}$ for all devices $j$ connected to $i$, where $N_i(t)$ is the number of devices $i$ connected to at iteration $t$. After all information is exchanged, for each connected device $j$, device $i$ assigns an aggregation weight $\alpha_{i, j}(t) = min\{e_{i, j}(t), e_{j, i}(t)\}$ for the local model $\mathbf{w}_j(t)$. It then assigns an aggregation weight $\alpha_{i, i}(t) = 1 - \sum\nolimits_{i \neq j}\alpha_{i, j}(t)$ for its own local model, and an aggregation weight $\alpha_{i, j}(t) = 0$ for all devices it does not have a connection to in iteration $t$. Denote the learning rate for device $i$ at iteration $t$ as $\eta_i(t)$. The local model for device $i$ for iteration $t+1$ is then updated as the weighted power-$p$ mean of the local models $\mathbf{w}_j(t)$ of all devices $j=1, \ldots, m$ with weights $\alpha_{i,j}(t)$:
\begin{equation}
\begin{aligned}
\label{eq:model}
\mathbf{w}_i(t+1) = \lbrack (\sum\nolimits_{j=1}^m \alpha_{i, j}(t) \mathbf{w}_j(t)^p ) - \eta d_i(t)\rbrack^{1/p},
\end{aligned}
\end{equation}
where $p$ is an integer. For vectors $\textbf{v}$, we define $\textbf{v}^C$ to be the element wise application of $x\to sgn(x)\cdot |x|^C$. Alternatively, let $\textbf{v}^C = \nabla h(\textbf{v})$ where $h(\textbf{v}) = \frac{1}{C+1}\lVert \textbf{v}\rVert_{C+1}^{C+1}$
In the case that the gradient step is zero, this makes $\mathbf{w}_i(t+1)$ the WPM of its neighbor's estimates at time $t$, which is different from the usual linear aggregation. 

Note that when $p=1$, the WPM aggregation function in our mechanism reduces to the typical weighted linear mean aggregation function commonly used by classical distributed machine learning systems:
\begin{equation}
\label{eq:classical}
\mathbf{w}_i(t+1) = \sum\nolimits_{j=1}^m \alpha_{i, j}(t) \mathbf{w}_j(t) - \eta(t) d_i(t),
\end{equation}
such as~\cite{candecentralizedoutperform,hegedHus2019gossip,nedic2020distributed,neglia2020decentralized}.

\begin{algorithm}[t]
\caption{The training procedure of WPM mechanism for each device.}
\label{alg:new}
\begin{algorithmic}[1]
\STATE \textbf{Input: }  the dataset $D_i$, the initial model parameters $\mathbf{w}_i(0)$, and the learning rate $\eta_i(t)$ for each device $i$, $t=0, \ldots, T$.
\STATE \textbf{Output: } The final model parameters of all devices after $T$ iterations $\mathbf{w}_i(t)$.
\FOR{$t=0$ to $T$}
\FOR{$i=1$ to $m$}
\STATE Compute the local gradient $d_i(t)$.
\STATE Check wireless connections to other devices and count the number of connected devices as $N_i(t)$.
\STATE Send $\mathbf{w}_i(t)$ and $e_{i,j} (t)=1/(N_i(t)+1)$ to all connected devices $j$.
\STATE Assign $\alpha_{i,j}(t)$ for all devices $j$, $j=1, \ldots, m$.
\STATE Compute $\mathbf{w}_i(t+1)$ using Equation (\ref{eq:model}).
\ENDFOR

\ENDFOR
\STATE Return $\mathbf{w}_i(T)$.
\end{algorithmic}
\end{algorithm}

\section{Convergence Analysis}
\label{sec:analysis}
Firstly, we present several assumptions needed to prove convergence.
\begin{assumption}
\label{assumption:network}
The network $G= (V,E_t)$ and the weight matrix $P(t)$ satisfy the following~\cite{yuan2020distributed} 
\begin{itemize}
    \item $P(t)$ is doubly stochastic for all $t\geq 1$, that is $\sum_{j=1}^m [P(t)]_{ij}=1$ and $\sum_{i=1}^m [P(t)]_{ij}=1$, for all $i,j\in V$.
    \item There exists a scale $\zeta>0$, such that $[P(t)]_{ii}\geq \zeta$ for all $i$ and $t\geq 1$, and $[P(t)]_{ij}\geq \zeta$, if $\{i,j\}\in E_t$.
    \item There exists an integer $B\geq 1$ such that the graph $(V, E_{kB+1}\cup \cdots \cup E_{(k+1)B})$ is strongly connected for all $k\geq 0$.
\end{itemize}
\end{assumption}
This assumption is very common in the literature~\cite{nedic2009distributed,li2017distributed,nedic2020distributed}
, and is a basic requirement that guarantees a necessary level of communication necessary to optimise based on the loss functions at each device.

\begin{assumption}
Functions $f_{i,t}$  for $1\leq i\leq m$ and $t\geq 0$ are convex. Additionally, they are $G_l-$ Lipschitz~\cite{yuan2020distributed} . That is 
\begin{equation}
    \lVert f_{i,t}(x) - f_{i,t}(y)\rVert \leq G_l \lVert x-y\rVert.
\end{equation}
Additionally, this implies that the gradient of $f_{i,t}$ is always $\lVert \nabla f_{i,t}(x)\rVert \leq G_l$. 
\end{assumption}


Let $w:\mathbb{R}^n \to \mathbb{R}$ be a function that is strictly convex, continuously differentiable, and defined on $\mathbb{R}^n$. Then, the Bregman divergence~\cite{beck2003mirror} is defined as 
\begin{equation}
   D_w(x,y) = w(x)-w(y)-\langle \nabla w(y), x-y \rangle. 
\end{equation}
Since $w$ is convex, $D_w(x,y)\geq 0$, and the Bregman divergence also functions as a distance measuring function.
In the process described in this paper, we take Mirror Descent steps in the special case where $w(x) = \frac1{p+1} \lVert x\rVert_{p+1}^{p+1}$ is the function used to create the distance measuring Bregman divergence function.

Mirror descent has its origins in classical convex optimisation~\cite{nemirovski1979efficient}, while follow the regularised leader goes back to the work of~\cite{gordon1999regret}. The mirror descent algorithm (MDA) was introduced by Nemirovsky and Yudin for solving convex optimization problems~\cite{nemirovskij1983problem}. This method exhibits an efficiency estimate that is mildly dependent in the decision variables dimension, and thus suitable for solving very large scale optimization problems. 

Our algorithm is a natural extension of Mirror Descent to DML~\cite{nemirovskij1983problem,beck2003mirror}. 
Mirror Descent also strives to solve $\min_x f(x)$ for convex $f:\mathbb{R}^n\to \mathbb{R}$. It generates a sequence $\{x_i\}_i$ of estimates where after calculating $x_t$, it generates the estimator function $f^{(t)}(x) = f(x_t) + \eta \nabla f_{i,t}(x_t)^T (x-x_t) + D_w(x,x_t)$, and lets $x_{t+1} = \argmin f^{(t)}(x)$. This satisfies $\nabla w(x_{t+1}) = \nabla w(x_t) - \eta \nabla f(x_t)$, which is used in our algorithm in Equation (\ref{eq:mirror}).

Mirror Descent also serves as a generalization of classical gradient descent. The case when $w(x)=\frac12 \lVert x\rVert^2$ and $D_w(x,y) = \lVert x-y\rVert^2$ gives $x_{t+1} = \argmin f^{(t)}(x) = x_t -\eta \nabla f(x_t)$.

In our algorithm, each step begins by taking a WPM as described in Equation (\ref{eq:rewrite}). Next, $y_{i,t+1}$ is a component wise WPM of $z_{i,t}$ since $\nabla w(x)$ raises each element to the power of $p$. However, if one were to then take standard gradient steps instead of the specific Mirror Descent step in Equation \ref{eq:mirror}, it would not converge to the optimal value $x^{\ast}$. Instead, we discovered that it was necessary to use Mirror Descent where the map mapping function satisfied $\nabla w(x) = x^p$ with termwise exponentiation.

The usage of the WPM is motivated by the idea that as $p$ increases, the consensus distance decreases.



The \sys{} model is the following
\begin{equation}
\begin{aligned}
&\overline{\mathbf{w}}_i^{(t)} = \left(\sum_{j=1}^m \alpha_{ij}\cdot (\mathbf{w}_j^{(t)})^p \right)^{\frac1p}. \\
&\mathbf{w}_i^{(t+1)} = \left(\left(\overline{\mathbf{w}}_i^{(t)}\right)^p - \eta \nabla f_i (\overline{\mathbf{w}}_i^{(t)})\right)^{\frac1p} \\
&\qquad\quad =\left(\sum_{j=1}^m \alpha_{ij} \left(\mathbf{w}_j^{(t)}\right)^p - \eta \nabla f_i (\overline{\mathbf{w}_i}^{(t)})\right)^{\frac1p}. 
\end{aligned}    
\end{equation}
Firstly, the $\overline{\textbf{w}}_i^{(t)}$ is calculated by taking a WPM of the models at the neighbors of node $i$. Then, $\textbf{w}_{i}^{(t+1)}$ is formed by taking a Mirror Descent step from $\overline{\textbf{w}}_i^{(t)}$. Additionally, in the experiments, a constant $\eta_0$ was chosen and for each value of $p$, $\eta$ was set as $\eta_0^{1+\frac12 p}$.

We rewrite this with $y_{i,t} = \overline{\textbf{w}_i}^{(t)}$ and $z_{i,t} = \textbf{w}_i^{(t+1)}$ for all $i,t$. Then, the process becomes
\begin{equation}
\label{eq:rewrite}
\nabla w(y_{i,t+1}) = \sum_{j=1}^m P(t)_{ij} \nabla w (z_{i,t}).
\end{equation}
\begin{equation}
\label{eq:mirror}
\nabla w(z_{i,t}) = \nabla w(y_{i,t}) - \eta \nabla f_{i,t} (y_{i,t}),
\end{equation}
where $w(\textbf{x}) = \frac1{p+1} \lVert \textbf{x} \rVert_{p+1}^{p+1}$.

\begin{claim}[Lower Bound]
\label{claim:lower}
For $x,y\in \mathbb{R}$,
 \[|x^p-y^p| \geq \frac1{2^{p-1}}|x-y|^p{.}\] 
 and as a corollary for $\textbf{x},\textbf{y}\in \mathbb{R}^n$, $\lVert \nabla w(\textbf{x}) - \nabla w(\textbf{y})\rVert_1 \geq \frac1{2^{p-1}} \lVert \textbf{x}-\textbf{y}\rVert_p^p$.
\end{claim}
The proof is in the Appendix.

\begin{claim}
\label{claim:stronk}
There exists a $\sigma_p\geq \frac{1}{2^{p-1}}$ for every positive integer $p$ such that if $w(x)= \frac1{p+1} \lVert x\rVert^{p+1}$, then
\[D_w(a,b) \geq \frac{\sigma_p}{p+1} \lVert a-b\rVert_{p+1}^{p+1}{.}\]
\end{claim}
The proof is in the Appendix.


To begin, in Lemma 1, we demonstrate that under WPM aggregation, the values at all nodes are kept close together.
\begin{lemma}
\label{lemma:consensus}
For all $i,t$,
\begin{equation}
   \lVert y_{i,t} - \overline{y_t}\rVert \leq \sqrt[p]{2^{p-1} \left(\vartheta \kappa^{t} \sum_{k=1}^m \lVert \nabla w(y_{k})\rVert +  \frac{m\eta G_l}{1-\kappa}\right)}. 
\end{equation}
where $\vartheta = \left(1-\frac{\zeta}{4m^2} \right)^{-2}$ and $\kappa = \left(1-\frac{\zeta}{4m^2} \right)^{\frac1B}$.
\end{lemma}
\begin{proof}
Under Assumption \ref{assumption:network}, Corollary 1 in \cite{nedic2008distributed} states that
\begin{equation}
\label{eq:expmix}
 |P(t,\tau)_{ij} - \frac1m | \leq \vartheta \kappa^{t-\tau}.
 \end{equation}

Note that we may write a general formula for $\nabla w(y_{i,t+1})$. Define the matrix $P(t,s) = P(t)P(t-1)\cdots P(s+1)P(s)$. Then,
\begin{equation}
\begin{aligned}
 &\nabla w(y_{i,t+1}) = \sum_{k=1}^m P(t,s)_{ik} \nabla w(y_{k,s})\\  &\qquad\qquad -\sum_{r=s}^t \sum_{k=1}^m P(t,r)_{ik} \eta \nabla f_{k,r}(y_{k,r}). 
\end{aligned}
\end{equation}

Note that $\nabla w(y_{i,t}) - \nabla w(\overline{y_r})$. Then, by applying the Triangle Inequality and Equation (\ref{eq:expmix}),
\begin{equation}
\begin{aligned}
    &\lVert \nabla w(y_{i,t+1}) - \nabla w(\overline{y_{t+1}})\rVert\\
    &\leq \sum_{k=1}^m \kappa^{t-s}\lVert \nabla w(y_{k,s})\rVert +\sum_{r=s}^t \sum_{k=1}^m \kappa^{t-r} \lVert \eta \nabla f_{k,r}(y_{k,r})\rVert\\
    &\leq \sum_{k=1}^m \kappa^{t-s}\lVert \nabla w(y_{k,s})\rVert + \sum_{r=s}^t \sum_{k=1}^m \vartheta \kappa^{t-r} \lVert \eta \nabla f_{k,r}(y_{k,r})\rVert\\
    &\leq \vartheta \kappa^{t-s} \sum_{k=1}^m \lVert \nabla w(y_{k,s})\rVert + m\eta G_l \sum_{r=s}^t \kappa^{t-r}\\
    &\leq \vartheta \kappa^{t-s} \sum_{k=1}^m \lVert \nabla w(y_{k,s})\rVert + m\eta G_l \cdot \frac{1}{1-\kappa}.
\end{aligned}
\end{equation}

To finish, we use Claim \ref{claim:lower} and plug $s=0$ in to get the stated bound.
\begin{equation}
\begin{aligned}
  &\lVert y_{i,t} - \overline{y_t}\rVert_p^p \leq \sqrt[p]{2^{p-1} \lVert \nabla w(y_{i,t}) - \nabla w(\overline{y_t})\rVert_1} \\
  &\qquad \leq \sqrt[p]{2^{p-1} \left(\vartheta \kappa^{t-s} \sum_{k=1}^m \lVert \nabla w(y_{k,s})\rVert_1 + m\eta G_l \cdot \frac{1}{1-\kappa}\right)}.    
\end{aligned}
\end{equation}
\end{proof}
\begin{theorem}
\label{theorem}
For all nodes $i$ and finishing time $T$, the regret satisfies
\begin{equation}
 \begin{aligned}
    &\overline{\textbf{Reg}_i}(T):= \frac1T \sum_{t=1}^T \sum_{j=1}^m f(\overline{y_{i,t}}) - f_{i,t}(x^{\ast}) \\
    & \leq \frac{m}{T} \frac1{\eta} D_w(x^{\ast},\overline{y_0}) +  m  2G_l \frac{p}{p+1}\sqrt[p]{\frac{\eta G_l}{\sigma_p}}\\
    &+2 G_l m \sqrt[p]{2^{p-1} (\vartheta \kappa^{t-s} \sum_{k=1}^m \lVert \nabla w(y_{k,s})\rVert_1 + m\eta G_l \cdot \frac{1}{1-\kappa})}.
\end{aligned}   
\end{equation}

\end{theorem}

\begin{proof}
Define $\nabla w(\overline{y_{t}}) = \frac1m \sum_{i=1}^m \nabla w(y_{i,t})$. 
Thanks to Lemma \ref{lemma:consensus}, it is guaranteed that $\overline{y_{t}}$ is close to $y_{i,t}$. This allows us to show the convergence to optimality of the comparatively simpler $\overline{y_t}$.
\begin{equation}
  \nabla w(\overline{y_{t+1}})=\nabla w(\overline{y_{t}}) - \eta \frac1m \sum_{i=1}^m \nabla f_{i,t}(y_{i,t}).  
\end{equation}
Then, by the convexity of $f_{i,t}$ and the definition of $\overline{y_t}$
\begin{equation}
\label{eq:mainlinebound}
 \begin{aligned}
    & \frac1m \sum_{i=1}^m f_{i,t}(y_{i,t}) - f_{i,t}(x) \leq \frac1m \sum_{i=1}^m \nabla f_{i,t}(y_{i,t})^T (y_{i,t}-x)\\
    &= \frac1m \sum_{i=1}^m \nabla f_{i,t}(y_{i,t})^T ((\overline{y_t}-x)+(y_{i,t}-\overline{y_t}))\\
    &= \frac1{\eta} (\nabla w(\overline{y_t}) - \nabla w(\overline{y_{t+1}}))^T (\overline{y_t}-x) \\
    & \quad + \frac1m \sum_{i=1}^m \nabla f_{i,t}(y_{i,t})^T (y_{i,t}-\overline{y_t})\\
    &\leq \frac1{\eta} (D_w(x,\overline{y_t}) + D_w(\overline{y_t},\overline{y_{t+1}}) - D_w(x,\overline{y_{t+1}}) \\
    & \quad + \frac1m G_l \sum_{i=1}^m \lVert y_{i,t} - \overline{y_t}\rVert.
\end{aligned}   
\end{equation}

Additionally, note that by Lipschitz, 
\begin{equation}
\label{eq:lipswitch}
\begin{aligned}
     &\frac1m \sum_{i=1}^m f_{i,t}(\overline{y_t}) - f_{i,t}(x)\\
     &\leq \frac1m \sum_{i=1}^m (f_{i,t}(y_{i,t}) - f_{i,t}(x) + G_l\lVert y_{i,t} - \overline{y_t}\rVert).  
\end{aligned}
\end{equation}
Thus, we are left to bound $D_w(\overline{y_t},\overline{y_{t+1}})$.
\begin{equation}
\label{eq:Dwytbound}
 \begin{aligned}
    &D_w(\overline{y_t}, \overline{y_{t+1}})= w(\overline{y_t})-w(\overline{y_{t+1}}) - \nabla w(\overline{y_{t+1}})^T (\overline{y_t}-\overline{y_{t+1}})\\
    &\leq \nabla w(\overline{y_t})^T(\overline{y_t}-\overline{y_{t+1}}) - \frac{\sigma_p}{p+1}\lVert \overline{y_t}-\overline{y_{t+1}}\rVert^{p+1} \\
    &\quad - \nabla w(\overline{y_{t+1}})^T (\overline{y_t}-\overline{y_{t+1}})\\
    &= (\nabla w(\overline{y_t})-\nabla w(\overline{y_{t+1}}))^T(\overline{y_t}-\overline{y_{t+1}}) \\
    &\quad - \frac{\sigma_p}{p+1}\lVert \overline{y_t}-\overline{y_{t+1}}\rVert^{p+1}\\
    &= (\eta \sum_{i=1}^m \frac1m \nabla f_{i,t}(y_{i,t}))^T (\overline{y_t}-\overline{y_{t+1}}) - \frac{\sigma_p}{p+1} \lVert \overline{y_t} - \overline{y_{t+1}}\rVert^{p+1}\\
    &\leq \eta G_l \lVert \overline{y_t} - \overline{y_{t+1}}\rVert - \frac{\sigma_p}{p+1}\lVert \overline{y_t} - \overline{y_{t+1}}\rVert^{p+1}\\
    &\leq \eta G_l \frac{p}{p+1} \sqrt[p]{\frac{\eta G_l}{\sigma_p}}.
\end{aligned}   
\end{equation}
Where the first inequality comes from Claim \ref{claim:stronk}, and the last inequality follows from $ax-bx^{p+1} \leq a \frac{p}{p+1} \sqrt[p]{\frac{a}{(p+1)b}}$ for $a,b\geq 0$. Combining the bounds in Equations (\ref{eq:mainlinebound}), (\ref{eq:lipswitch}), and (\ref{eq:Dwytbound}) gives
\begin{equation}
 \begin{aligned}
    &\sum_{t=0}^{T-1} \frac1m \sum_{i=1}^m f_{i,t}(\overline{y_t}) - f_{i,t}(x)\\
    &\leq \sum_{t=0}^{T-1} \frac1\eta (D_w(x,\overline{y_t}) + D_w(\overline{y_t}, \overline{y_{t+1}}) - D_w(x,\overline{y_{t+1}}))\\
    &+ \frac2m G_l \sum_{t=0}^{T-1} \sum_{i=1}^m \lVert y_{i,t} - \overline{y_t}\rVert\\
    &\leq \frac1{\eta} (D_w(x,\overline{y_0})-D_w(x,\overline{y_{T}}))+ \frac1{\eta} T \eta 2G_l \frac{p}{p+1}\sqrt[p]{\frac{\eta G_l}{\sigma_p}}\\
    &+2 G_l T \sqrt[p]{2^{p-1} (\vartheta \kappa^{t-s} \sum_{k=1}^m \lVert \nabla w(y_{k,s})\rVert_1 + m\eta G_l \cdot \frac{1}{1-\kappa})}.
\end{aligned}   
\end{equation}

Which easily manipulates to the final theorem result.
\end{proof}

\begin{figure*}[hb]
\setlength{\abovecaptionskip}{-1cm}
\setlength{\belowcaptionskip}{-1cm}
\centering
\subfigure[Accuracy with Different P]
{
    \label{figure:acc_diff_p_iid}
    \begin{minipage}[b]{\picRatio\linewidth}
        \centering
        \includegraphics[scale=\accLossScale]{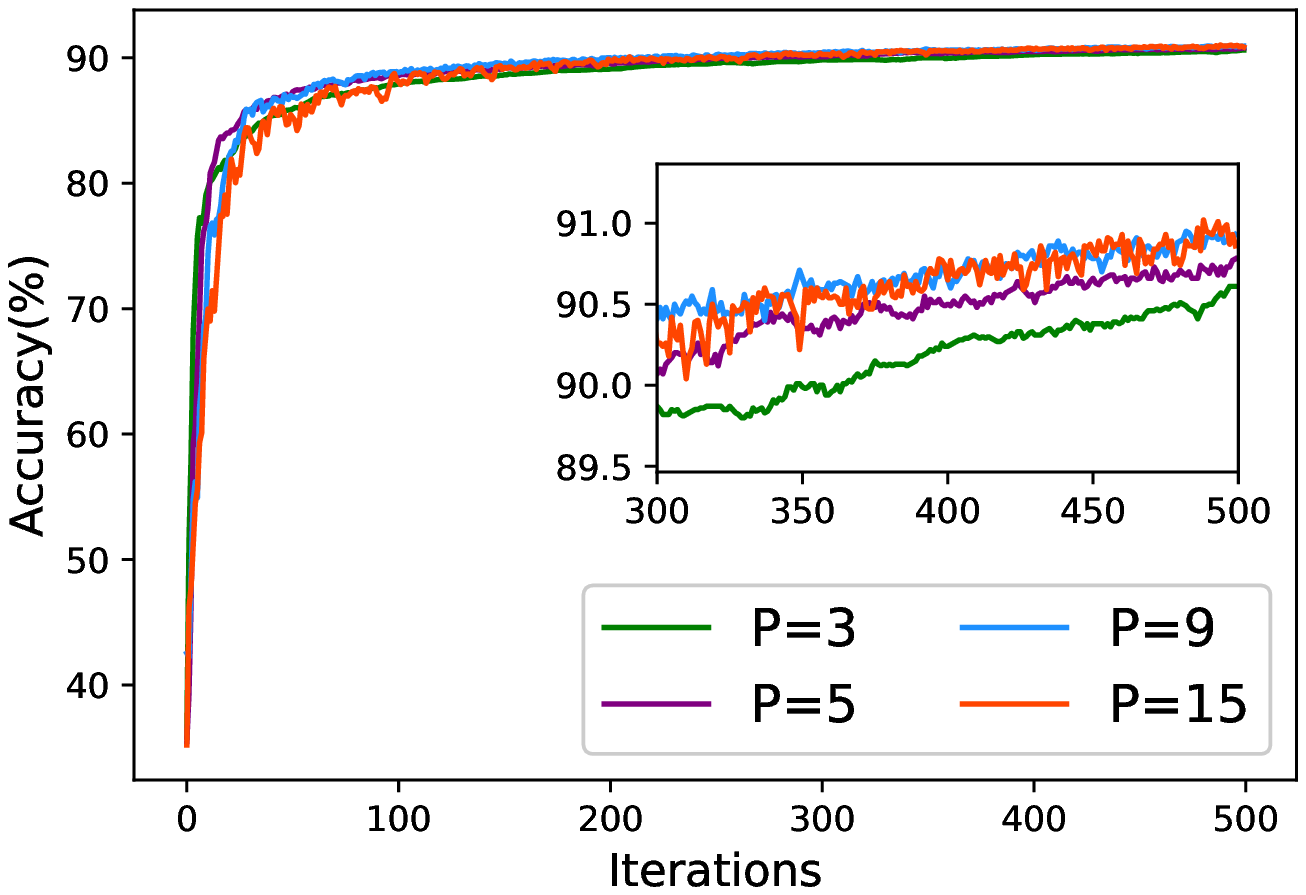} 
    \end{minipage}
}
\subfigure[Accuracy with 10 Devices]
{
    \setlength{\abovecaptionskip}{-1cm}
    \setlength{\belowcaptionskip}{-1cm}
    \label{figure:acc_size_10_iid}
    \begin{minipage}[b]{\picRatio\linewidth}
        \centering
        \includegraphics[scale=\accLossScale]{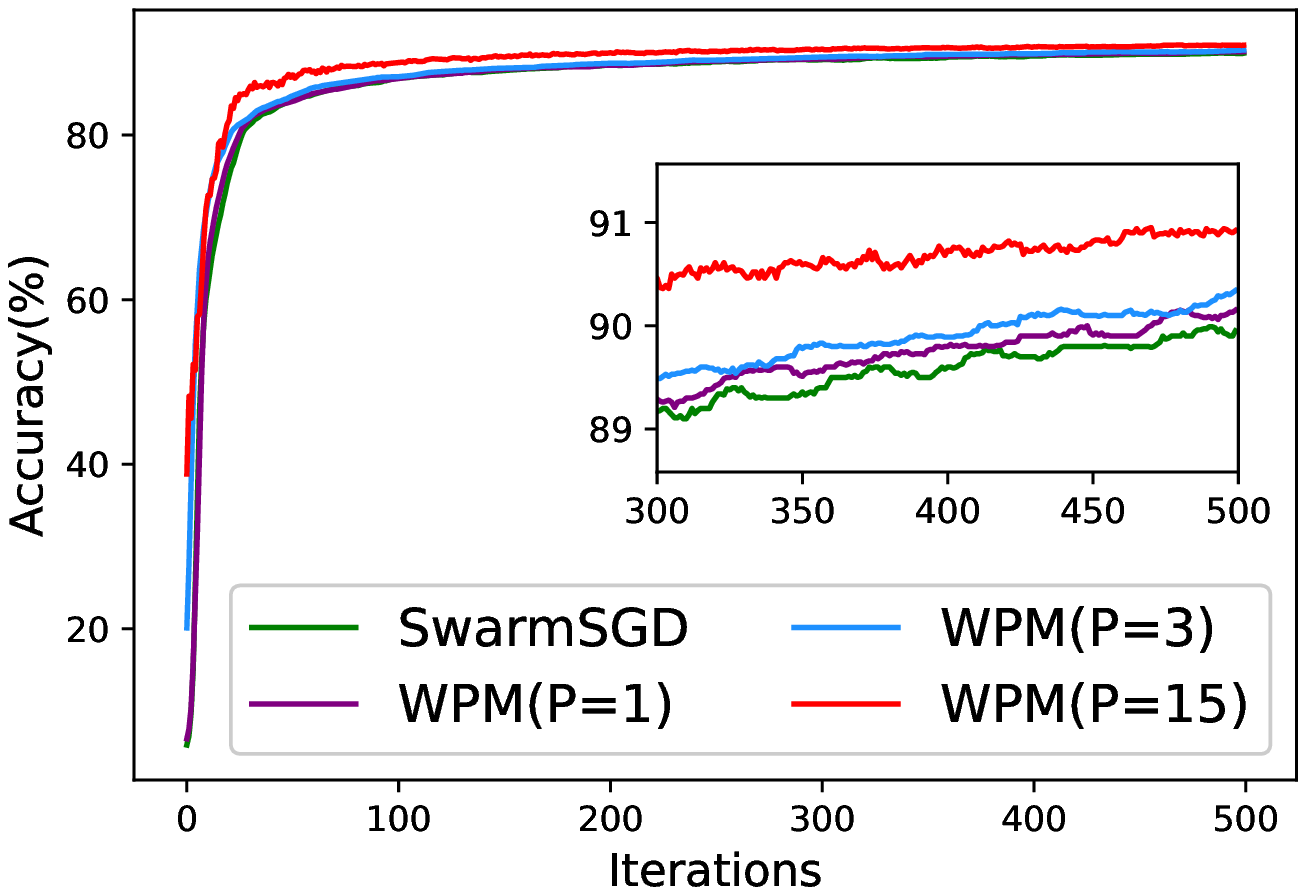} 
    \end{minipage}
}
\subfigure[Accuracy with 20 Devices]
{
    \setlength{\abovecaptionskip}{-1cm}
    \setlength{\belowcaptionskip}{-1cm}
    \label{figure:acc_size_20_iid}
    \begin{minipage}[b]{\picRatio\linewidth}
        \centering
        \includegraphics[scale=\accLossScale]{acc_LR_MNIST_iid.eps}
    \end{minipage}
}
\subfigure[Accuracy with 100 Devices]
{
    \setlength{\abovecaptionskip}{-1cm}
    \setlength{\belowcaptionskip}{-1cm}
    \label{figure:acc_size_100_iid}
    \begin{minipage}[b]{\picRatio\linewidth}
        \centering
        \includegraphics[scale=\accLossScale]{acc_LR_MNIST_iid.eps}
    \end{minipage}
}

\caption{
Accuracy Comparison on i.i.d Setting with Devices $\in \{10, 15, 20\}$ }
\label{fig:accuracy}
\end{figure*}

\begin{figure*}[hb]
\setlength{\abovecaptionskip}{-1cm}
\setlength{\belowcaptionskip}{-1cm}
\centering
\subfigure[Accuracy with Different P]
{
    \label{figure:acc_diff_p_noniid}
    \begin{minipage}[b]{\picRatio\linewidth}
        \centering
        \includegraphics[scale=\accLossScale]{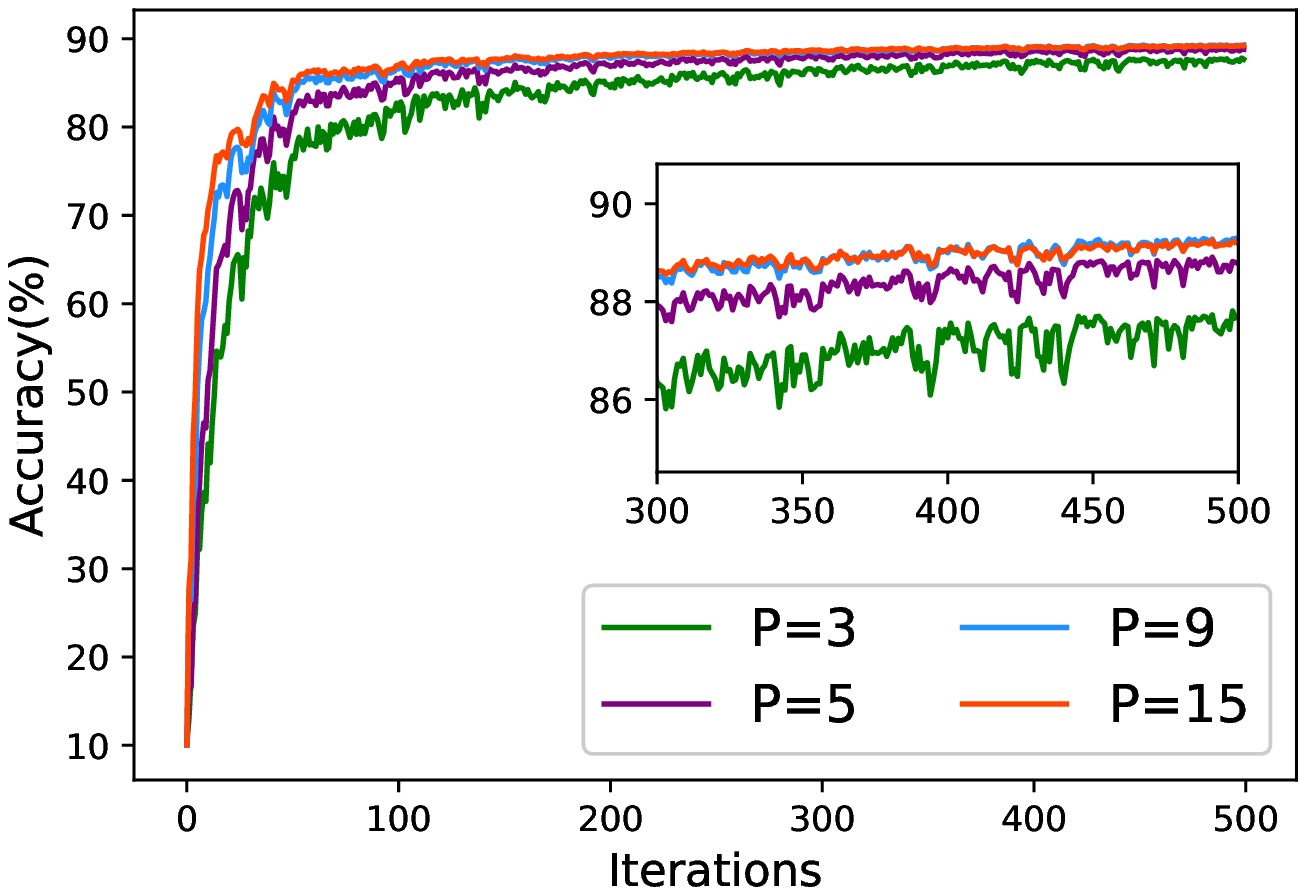} 
    \end{minipage}
}
\subfigure[Accuracy with 10 Devices]
{
    \label{figure:acc_size_10_noniid}
    \begin{minipage}[b]{\picRatio\linewidth}
        \centering
        \includegraphics[scale=\accLossScale]{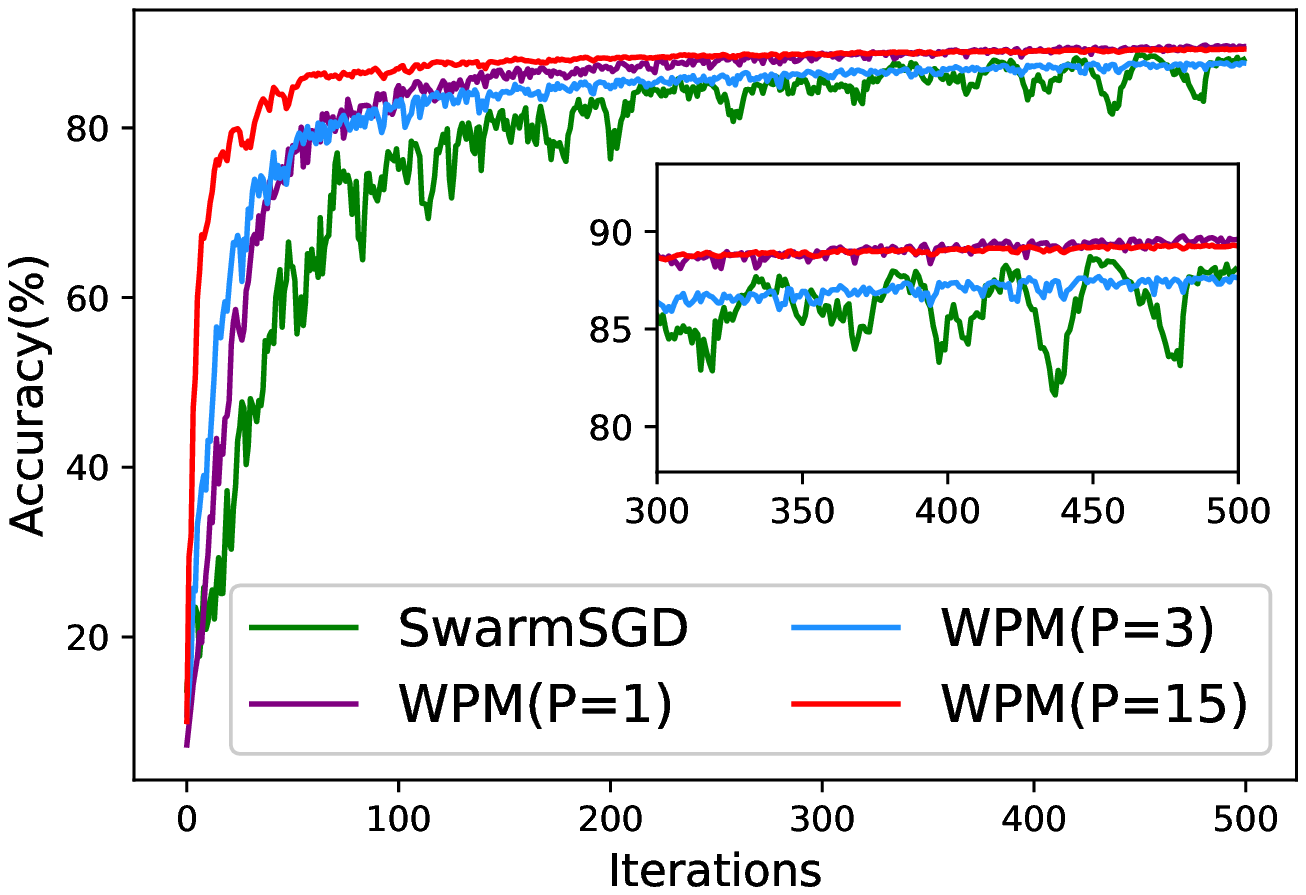} 
    \end{minipage}
}
\subfigure[Accuracy with 20 Devices]
{
    \label{figure:acc_size_20_noniid}
    \begin{minipage}[b]{\picRatio\linewidth}
        \centering
        \includegraphics[scale=\accLossScale]{acc_LR_MNIST_non-iid.eps}
    \end{minipage}
}
\subfigure[Accuracy with 100 Devices]
{
    \label{figure:acc_size_100_noniid}
    \begin{minipage}[b]{\picRatio\linewidth}
        \centering
        \includegraphics[scale=\accLossScale]{acc_LR_MNIST_non-iid.eps}
    \end{minipage}
}
\caption{
Accuracy Comparison on Non-i.i.d Setting with Devices $\in \{10, 15, 20\}$}
\label{fig:my_label}
\end{figure*}

\begin{figure*}[h]
\setlength{\abovecaptionskip}{-1cm}
\setlength{\belowcaptionskip}{-1cm}
\centering
\subfigure[Loss with Different P]
{
    \label{figure:loss_diff_p_iid}
    \begin{minipage}[b]{\picRatio\linewidth}
        \centering
        \includegraphics[scale=\accLossScale]{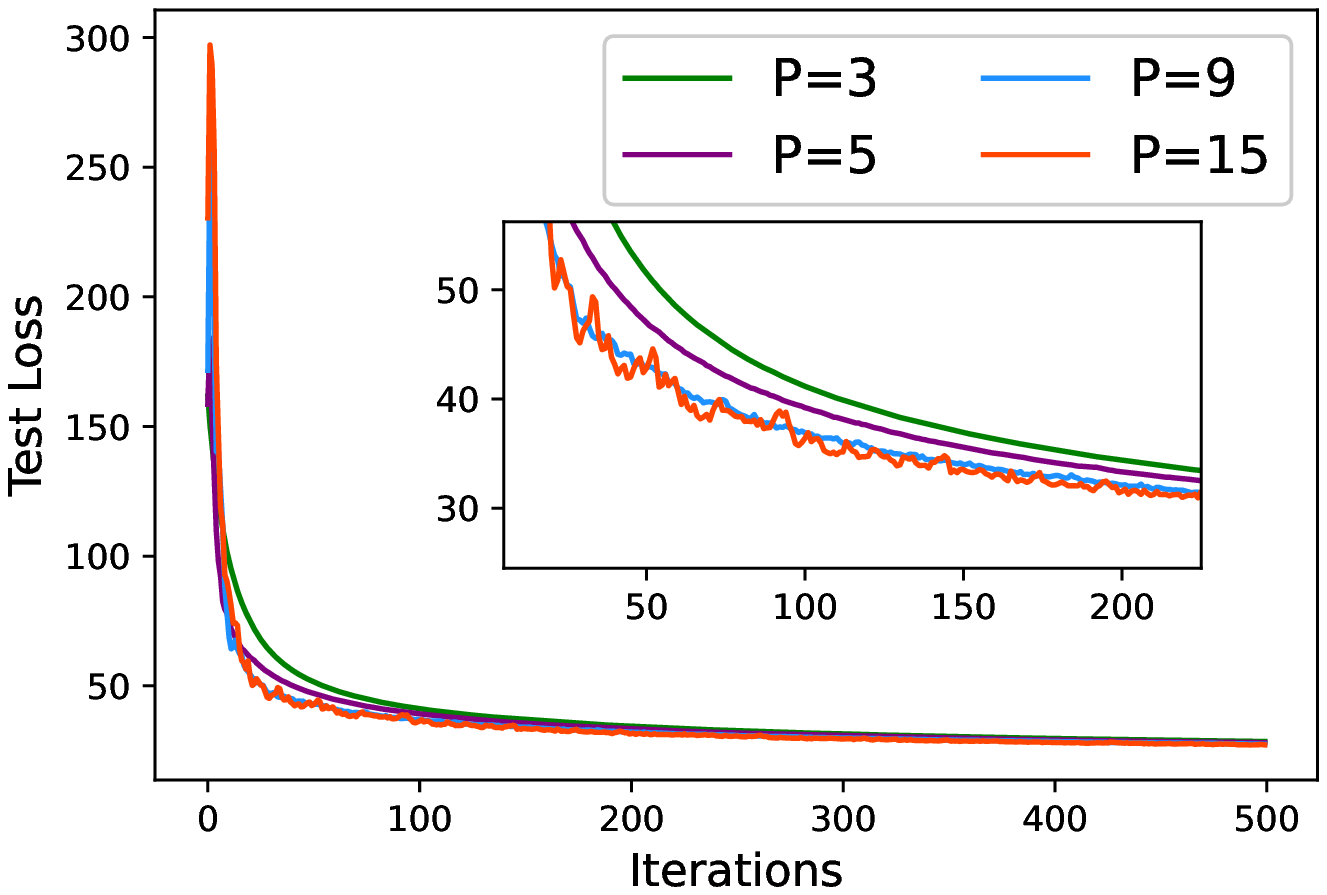} 
    \end{minipage}
}
\subfigure[Loss with 10 Devices]
{
    \label{figure:loss_size_10_iid}
    \begin{minipage}[b]{\picRatio\linewidth}
        \centering
        \includegraphics[scale=\accLossScale]{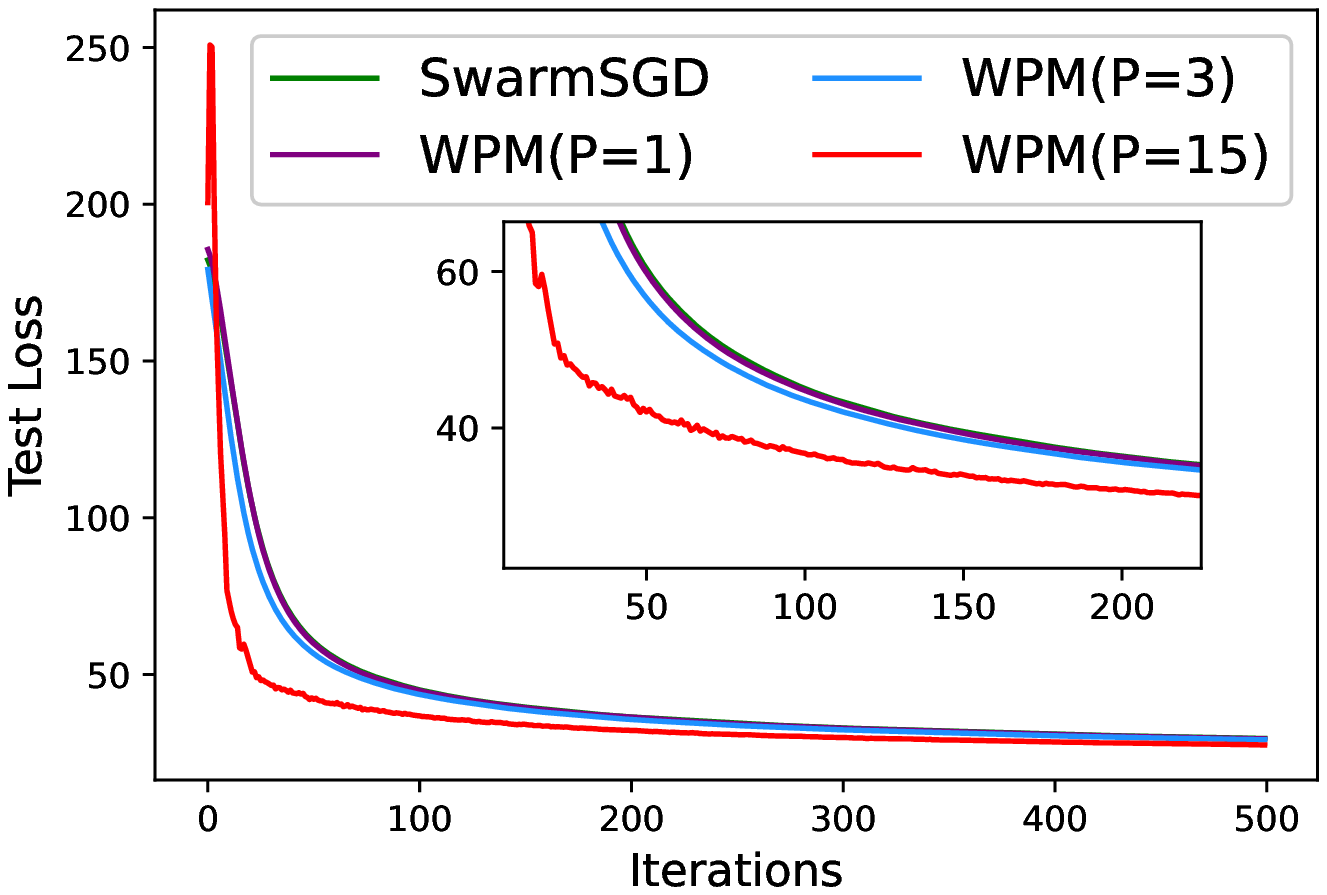}
    \end{minipage}
}
\subfigure[Loss with 20 Devices]
{
    \label{figure:loss_size_20_iid}
    \begin{minipage}[b]{\picRatio\linewidth}
        \centering
        \includegraphics[scale=\accLossScale]{loss_LR_MNIST_iid.eps}
    \end{minipage}
}
\subfigure[Loss with 100 Devices]
{
    \label{figure:loss_size_100_iid}
    \begin{minipage}[b]{\picRatio\linewidth}
        \centering
        \includegraphics[scale=\accLossScale]{loss_LR_MNIST_iid.eps}
    \end{minipage}
}
\caption{
Loss Comparison on i.i.d Setting with Devices $\in \{10, 15, 20\}$}
\label{fig:loss}
\end{figure*}

\begin{figure*}[h]
\setlength{\abovecaptionskip}{-1cm}
\setlength{\belowcaptionskip}{-1cm}
\centering
\subfigure[Loss with Different P]
{
    \label{figure:loss_diff_p_iid2}
    \begin{minipage}[b]{\picRatio\linewidth}
        \centering
        \includegraphics[scale=\accLossScale]{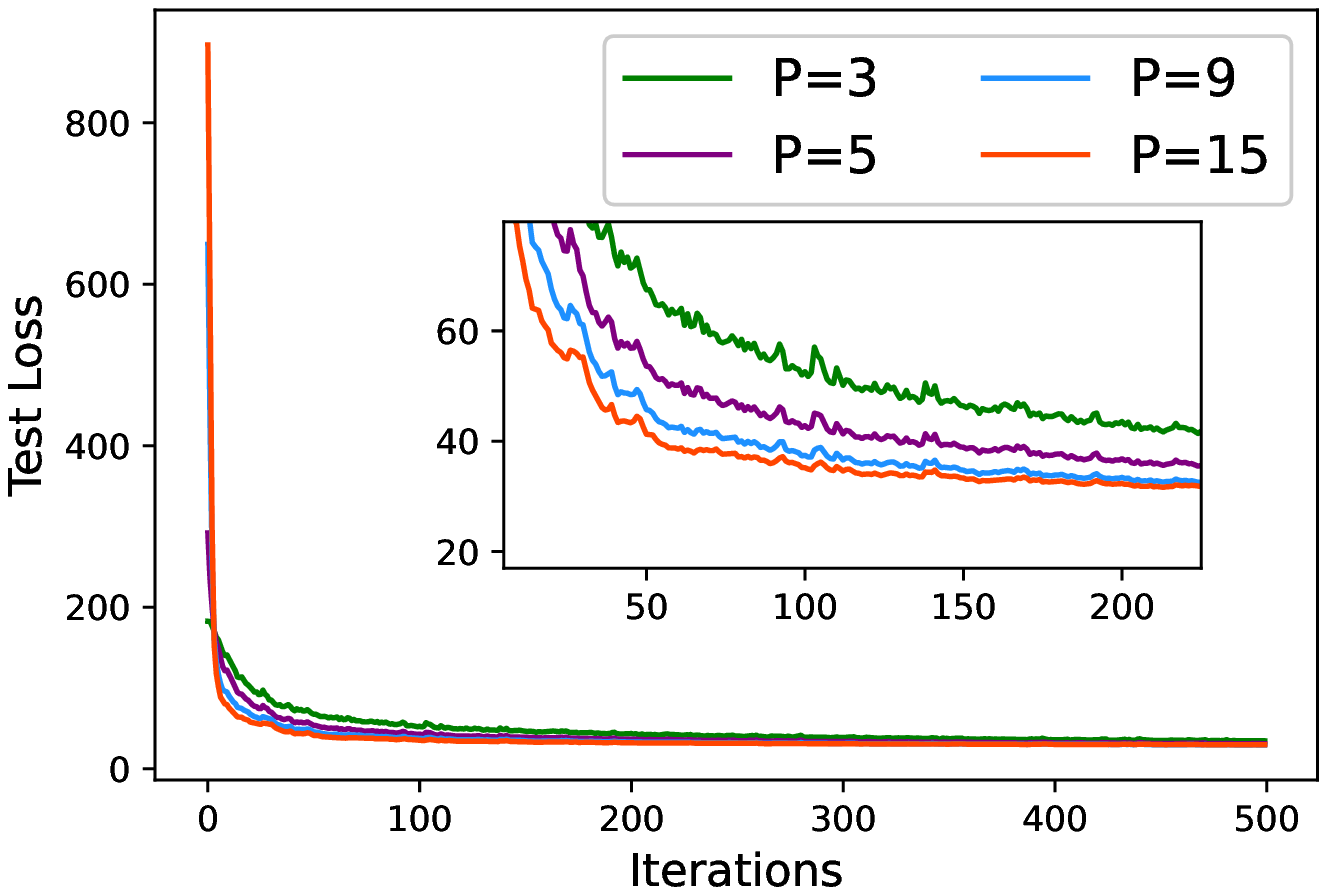} 
    \end{minipage}
}
\subfigure[Loss with 10 Devices]
{
    \label{figure:loss_size_10_iid2}
    \begin{minipage}[b]{\picRatio\linewidth}
        \centering
        \includegraphics[scale=\accLossScale]{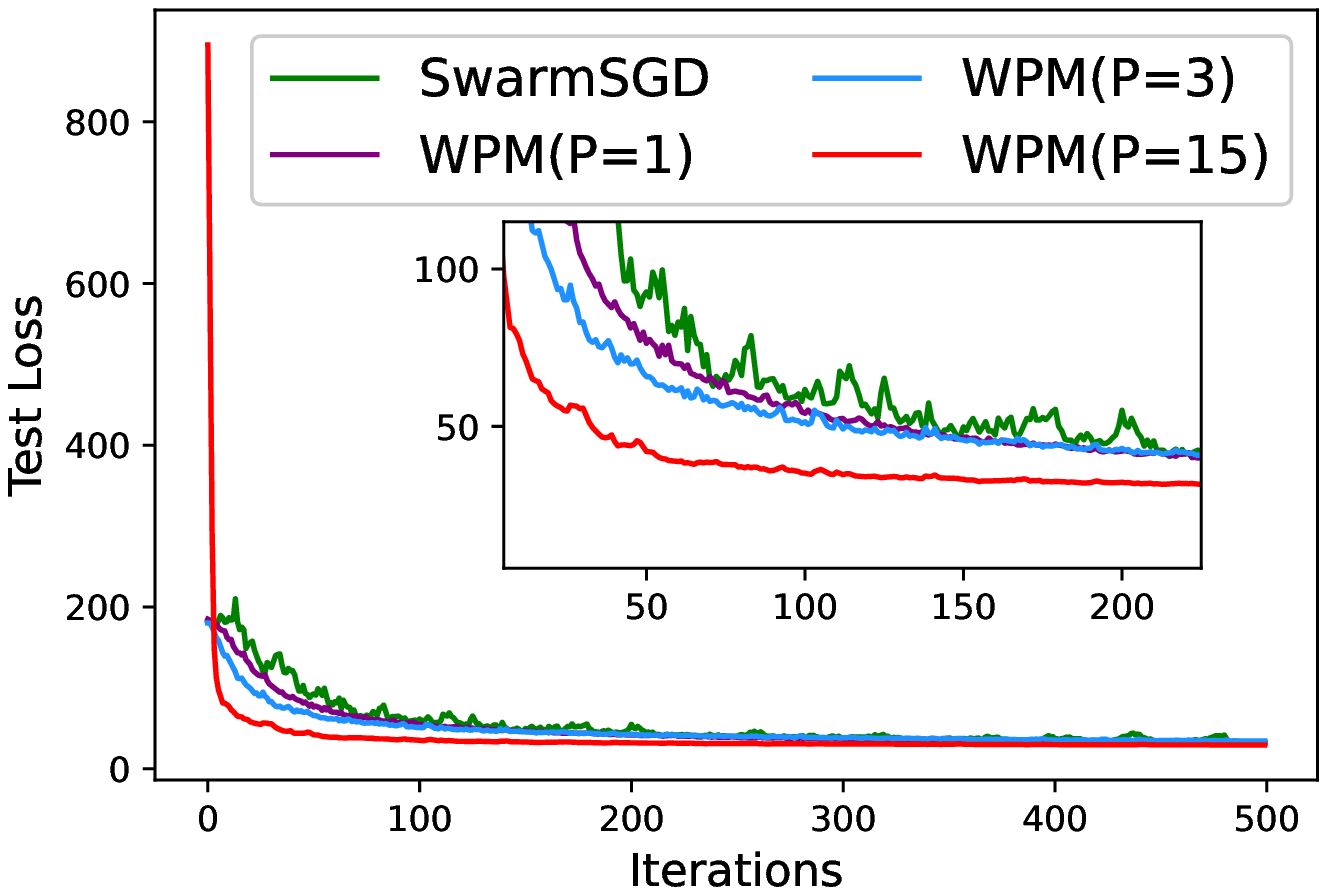}
    \end{minipage}
}
\subfigure[Loss with 20 Devices]
{
    \label{figure:loss_size_20_iid2}
    \begin{minipage}[b]{\picRatio\linewidth}
        \centering
        \includegraphics[scale=\accLossScale]{loss_LR_MNIST_non-iid.eps}
    \end{minipage}
}
\subfigure[Loss with 100 Devices]
{
    \label{figure:loss_size_100_iid2}
    \begin{minipage}[b]{\picRatio\linewidth}
        \centering
        \includegraphics[scale=\accLossScale]{loss_LR_MNIST_non-iid.eps}
    \end{minipage}
}
\caption{
Loss Comparison on Non-i.i.d Setting with Devices $\in \{10, 15, 20\}$}
\label{fig:loss2}
\end{figure*}

\begin{figure*}[!tbp]
\setlength{\abovecaptionskip}{-1cm}
\setlength{\belowcaptionskip}{-1cm}
\centering
\subfigure[Convergence Speed with Different P]
{
    \label{fig:con_diff_P_iid}
    \begin{minipage}[b]{\picRatio\linewidth}
        \centering
        \includegraphics[scale=\convergenceSpeedScale]{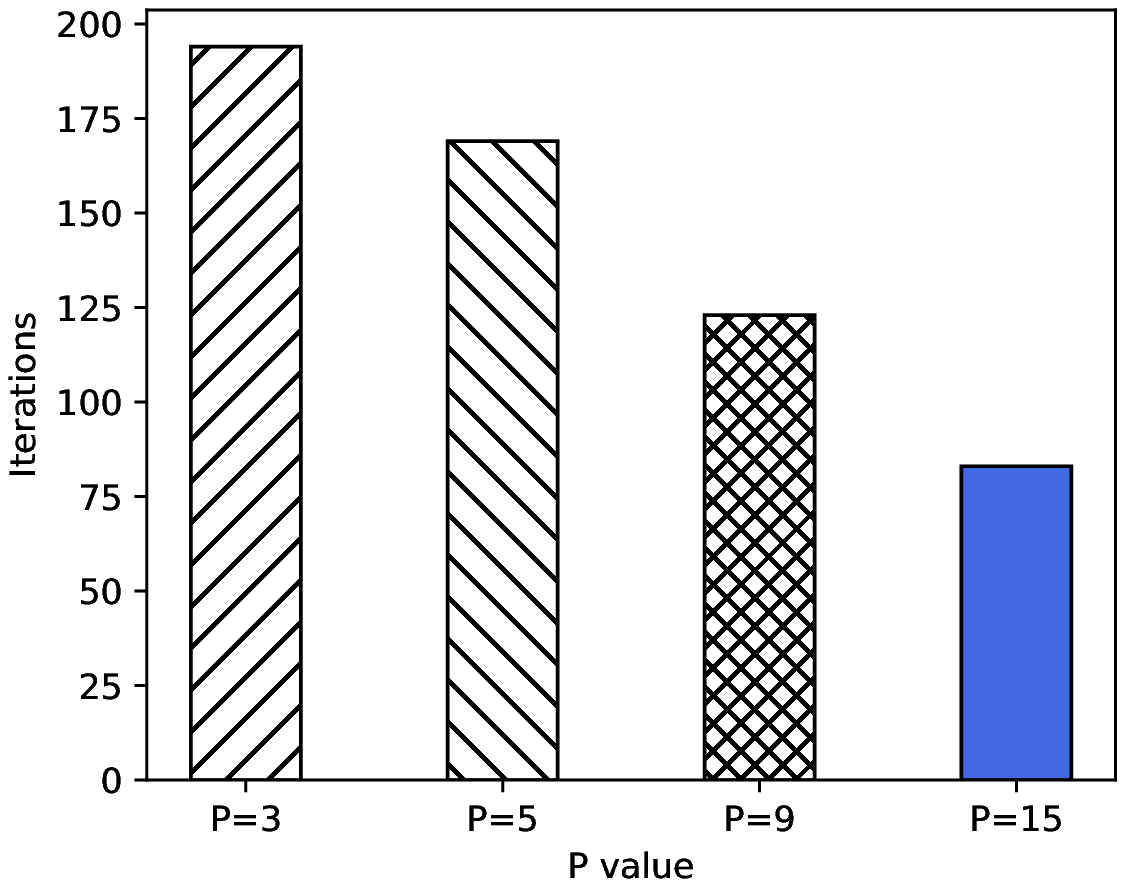}
    \end{minipage}
}
\subfigure[Convergence Speed with 10 Devices]
{
    \label{fig:con_size_10_iid}
    \begin{minipage}[b]{\picRatio\linewidth}
        \centering
        \includegraphics[scale=\convergenceSpeedScale]{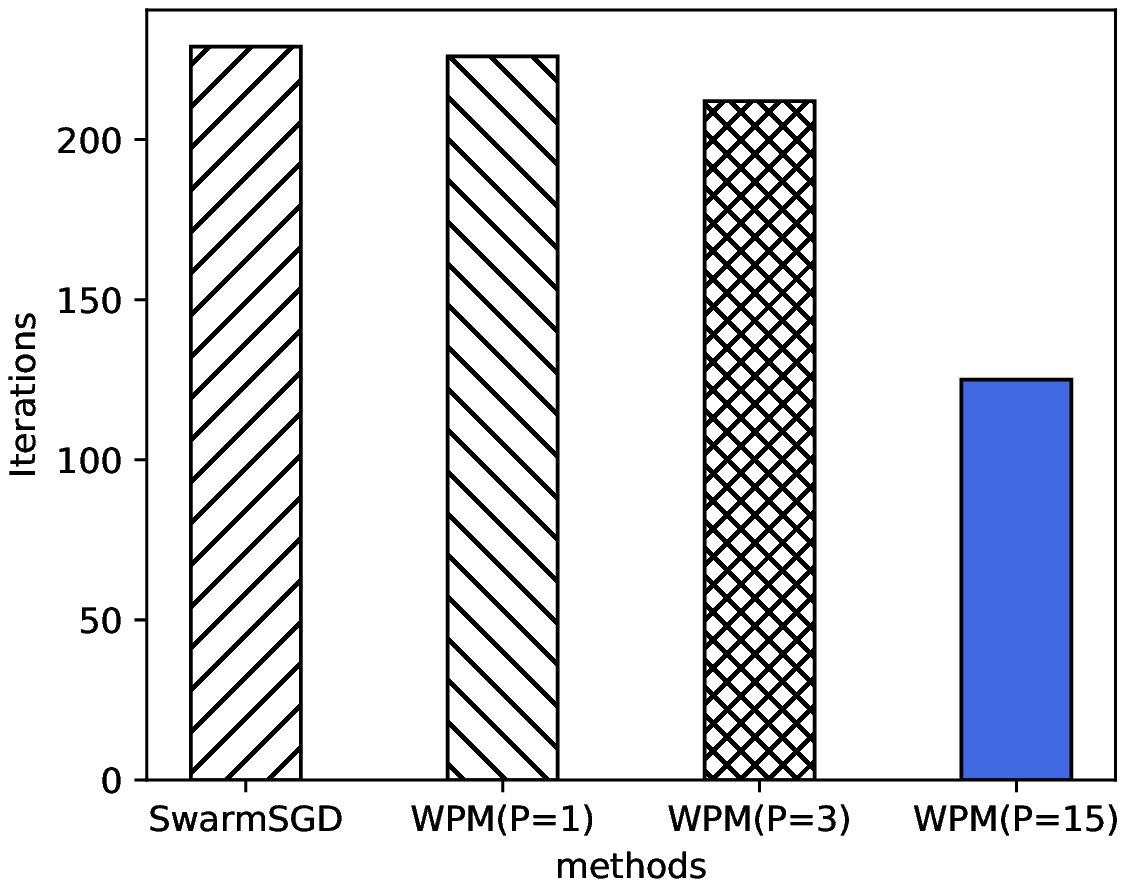} 
    \end{minipage}
}
\subfigure[Convergence Speed with 20 Devices]
{
    \label{fig:con_size_20_iid}
    \begin{minipage}[b]{\picRatio\linewidth}
        \centering
        \includegraphics[scale=\convergenceSpeedScale]{con-speed_LR_MNIST_iid.eps}
    \end{minipage}
}
\subfigure[Convergence Speed with 100 Devices]
{
    \label{fig:con_size_100_iid}
    \begin{minipage}[b]{\picRatio\linewidth}
        \centering
        \includegraphics[scale=\convergenceSpeedScale]{con-speed_LR_MNIST_iid.eps}
    \end{minipage}
}
\caption{
Convergence Speed Comparision on i.i.d Setting with Devices $\in \{10, 15, 20\}$}
\label{fig:convergence_speed}
\end{figure*}

\begin{figure*}[!tbp]
\setlength{\abovecaptionskip}{-1cm}
\setlength{\belowcaptionskip}{-1cm}
\centering
\subfigure[Convergence Speed with Different P]
{
    \label{fig:con_diff_P_noniid}
    \begin{minipage}[b]{\picRatio\linewidth}
        \centering
        \includegraphics[scale=\convergenceSpeedScale]{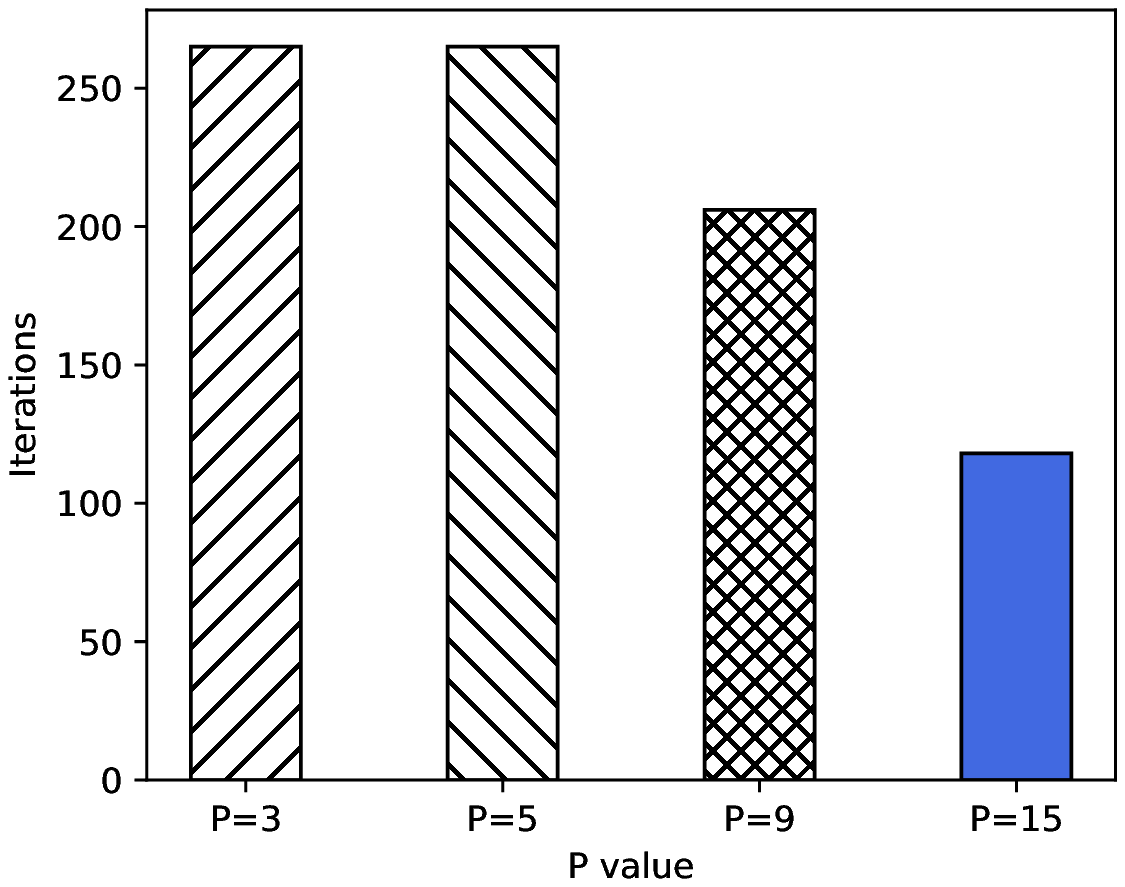} 
    \end{minipage}
}
\subfigure[Convergence Speed with 10 Devices]
{
    \label{fig:con_size_10_noniid}
    \begin{minipage}[b]{\picRatio\linewidth}
        \centering
        \includegraphics[scale=\convergenceSpeedScale]{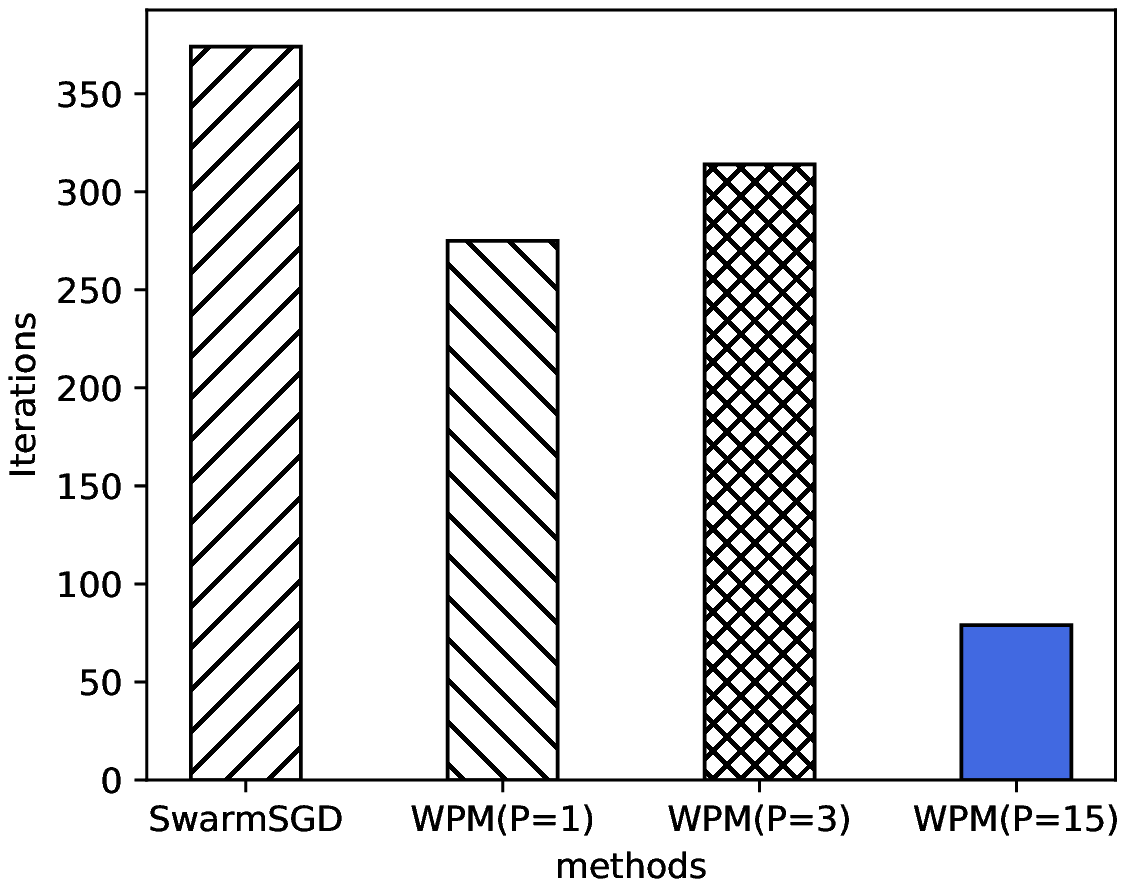} 
    \end{minipage}
}
\subfigure[Convergence Speed with 20 Devices]
{
    \label{fig:con_size_20_noniid}
    \begin{minipage}[b]{\picRatio\linewidth}
        \centering
        \includegraphics[scale=\convergenceSpeedScale]{con-speed_LR_MNIST_non-iid.eps}
    \end{minipage}
}
\subfigure[Convergence Speed with 100 Devices]
{
    \label{fig:con_size_100_noniid}
    \label{fig:con}
    \begin{minipage}[b]{\picRatio\linewidth}
        \centering
        \includegraphics[scale=\convergenceSpeedScale]{con-speed_LR_MNIST_non-iid.eps}
    \end{minipage}
}
\caption{
Convergence Speed Comparison on Non-i.i.d Setting with Devices $\in \{10, 15, 20\}$}
\label{fig:my_labe2l}
\end{figure*}

\section{Evaluations}
\label{sec:evaluations}

\subsection{Experimental Setup}
\label{subsec:settings}
The experimental platform is composed of 8 Nvidia Tesla T4 GPUs, 4 Intel XEON CPUs and 256GB memory. We build the experimental environment based on Ray \cite{ray2018}, an open-source framework that provides universal API for building high performance distributed applications.

\para{{Topology.}}
\label{subsubsec:topology}
We randomly generated fixed per-iteration-changing topology sequences to simulate the time-varying networks. Each element of a sequence is a randomly generated symmetric adjacent matrix of undirected graph topology.
We use $density$ to measure the proportion of non-zero elements in upper triangle of the adjacent matrix (\ie, the proportion of available connections in the topology). 
The generated topology is a fully connected topology when the $density = 1$, or a ring topology when the $density = 0.1$.

\para{Datasets and Models.}
\label{subsubsec:datasets_models}
We performed our experiment on the classic MNIST \cite{lecun1998gradient} and CIFAR-10 \cite{krizhevsky2009learning} datasets. 
Both are constructed for image classification tasks.
We use two typical convex models, the Logistic Regression (LR) model \cite{hosmer2013applied} and the support vector machines (SVM) \cite{hu2020robust} for our experiment. Specifically, We select the LR model for the MNIST, and SVM for CIFAR-10. 

\para{Baselines.}
\label{subsubsec:Baselines_Setup}
We compared the \sys{} with $P=1$ methods and SwarmSGD \cite{2019SwarmSGD} over time-varying topologyies. In order to provide a fair comparison between all three methods, we use the same dynamic topology sequences to conduct the experiment. And we made some minor adjustments for each baseline method.
\begin{itemize}
    \item For the {$P=1$}, we define it as a class of methods using a linear aggregation function~\cite{2017DPSGD,nedic2008distributed, neglia2020decentralized}.
    All of these methods are a special case of \sys{}.
    \item For the \textbf{SwarmSGD}, we set the number of local SGD updates equal to 1, where the selected pair of devices perform only one single local SGD update before aggregation.
\end{itemize}
\para{Data Partition.}
\label{subsubsec:datasets_partition}
We consider two data distributions scenarios \cite{pmlr-v139-esfandiari21a}.
For the IID setting, the training dataset is equally divided into $M$ random sub-datasets, and assigned to $M$ devices. 
For the Non-IID setting, we define an extreme case where the data on different devices are generated from different distributions (\ie{} each device hold only one label of data).



\para{{Metrics.}}
\label{subsubsec:performance_metrices}
To measure the performance of WPM, we study three metrics in the experiments.
\begin{itemize}

	\item \textbf{Model Accuracy.} This measures the proportion between the amount of incorrect data by the model and that of all data. Assume the number of samples at each device $i$ is $\chi_i$ and the number of right measured samples is $\chi_i(r)$. The accuracy of the \sys{} is measured as $\frac{\chi_i(r)}{\chi_i}$ at device $i$. The average accuracy is computed as $\frac{1}{\mathcal{M}} \sum_{i=1}^m \frac{\chi_i(r)}{\chi_i}$.
	\item \textbf{Test Loss.} This metric measures the quantification difference of probability distributions between model outputs and observation results. 
	
	\item \textbf{Convergence Speed.} 
	We record the accuracy of methods at each iteration and then relate the accuracy to the number of iterations to study how many iterations are required to achieve a specific loss or accuracy. 
	\end{itemize}

\para{\textbf{Hyperparameters.}}
\label{subsubsec:hyperparameters_setting}
For all of experiments, the learning rate and batch size are fixed at $0.01$ and $128$.
We randomly generate dynamic topology sequences with devices of 10, 15 and 20, and conduct the evaluation of total 500  iterations.

\subsection{Experiment Results}
\label{subsec:SExperiment Results}


\para{Vertical Comparison.}
\label{subsubsec:different_P_comparison}
To study how different $P$ affects the performance of WPM over time-varying networks. We evaluate WPM with different $P \in \{3, 5, 9, 15\}$, respectively, and choose LR model on MNIST dataset. 
We empirically choose $P=15$ as the max $P$ value.
Due to the operation of power mean with $P$, it results in the model parameters extremely small before gradient descent, which may lead to the vanishing gradient problem. Moreover, the computation complexity exponentially increased with larger $P$. At the same time, we found the computation cost unacceptable after $P>15$ in the experiment.

In the Figure \ref{fig:con_diff_P_iid} and Figure \ref{fig:con_diff_P_noniid}, results show that WPM has better performance in all metrics as the $P$ increases.
The larger $P$ averagely accelerated the convergence speed up to 30\%, at the same time slightly improved the model accuracy and loss. With the extreme Non-IID setting, the convergence speed significantly accelerated up to 62.03\% than the baseline.

\para{Horizontal Comparison.}
\label{subsubsec:Horizontal Comparison}
We use \sys{} ($P=3$) and \sys{} ($P=15$) compare with $P=1$ and SwarmSGD.
Results in Figure \ref{fig:con_size_10_iid} and Figure \ref{fig:con_size_10_noniid} indicate that \sys{} ($P=15$) achieves the fastest convergence speed in IID and Non-IID distributions. 
And also slightly improved accuracy and loss as shown in Figure \ref{figure:acc_size_10_iid} and Figure \ref{figure:acc_size_10_noniid}. 
Due to the extreme distinct Non-IID data setting, all methods perform worse than IID setting. This also caused the sharply loss increase for \sys{} at the beginning of the training process. However, \sys{} ($P=15$) still keep the best accuracy and loss, while other baselines especially the SwarmSGD struggles significantly because its less model aggregation rounds.

\para{Scalability.}
\label{subsubsec:Scalability_Comparison}
To analyze the performance and scalability of all methods under different scales, each method is performed with topology sequences of 10, 20, and 100 devices.
The results are illustrated in Figure \ref{fig:con_size_20_iid}, Figure \ref{fig:con_size_100_iid}, Figure \ref{fig:con_size_20_noniid}, Figure \ref{fig:con_size_20_noniid} and Figure \ref{fig:con_size_100_noniid}. 
Vertically, as the number of devices increased to 20, the convergence speed of \sys{} is accelerated by 10.13\% with IID setting, 13.74\% with Non-IID setting, while keep a similar accuracy. 
Horizontally, when comparing with DPSGD under the scale of 20 devices, \sys{} converge 25.26\% faster with IID setting, and significantly 57.99\% faster with Non-IID setting. We present all the detailed data in Table \ref{tab:scale compare table} of Appendix.

\para{\textbf{Found Lessons.}}
\label{subsubsec:hyperparameters_setting2}
We found our proposed \sys{} is very sensitive to the learning rate with mirror gradient descent. 
We use a learning rate scaling factor to keep the corresponding learning rate at the same magnitude with model parameters.
In our previous experiments, when we use an inappropriate learning rate scaling factor liking $  \frac{\mid(\overline{x}_i^t)^p \mid}{\mid f(\overline{x}_i^t) \mid}  $ for mirror descent. 
The WPM could not work under an inadequate learning rate.

		

\section{Conclusion}
\label{sec:con}
We study the important problem of DML over time-varying networks. We propose a novel non-linear mechanism that takes weighted power-$p$ means on the local models, where $p$ is a positive odd integer. To our best knowledge, It is the first time to do non-linear aggregation in DML. We rigorously prove its convergence. Experimental results show that our mechanism significantly outperforms the common linear weighted aggregation over time-varying networks, in terms of convergence speed and scalability.

\bibliography{ref}
\bibliographystyle{IEEEtran}

\clearpage
\onecolumn
{
}


\section{Appendix}
\label{sec:appendix}

\subsection{Proof of Claim~\ref{claim:lower} (Lower Bound)}
\begin{proof}
Substitute $a=x-y, b=x+y$. Then, this becomes
\begin{equation}
 2^p \left|\left(\frac{a+b}{2}\right)^p + \left(\frac{a-b}{2}\right)^p\right|\geq 2|a|^p.   
\end{equation}
Without loss of generality set $a>0$, and then it suffices to verify
\begin{equation}
 (a+b)^p + (a-b)^p \geq 2a^p.   
\end{equation}
This final inequality is true because after cancellation, the only terms left on the LHS are of the form $a^{p-2k}b^{2k}$, so the inequality is true.
\end{proof}

\subsection{Proof of Claim~\ref{claim:stronk}}

\begin{proof}
$\sigma_p = \frac{1}{2^{p-1}}$ works. 
Holding $b$ constant, the derivative of the LHS with respect to $a$ is $a^p - b^p$ and the derivative of the RHS with respect to $a$ is $\sigma_p (a-b)^p$. By the previous Lemma, this is true with $\sigma_p = \frac{1}{2^{p-1}}$, since the magnitude of the LHS's derivative is larger, and for $a<b$ both derivatives are negative, and for $a>b$ both derivatives are positive.
\end{proof}

\subsection{Additional Experimental Results}
\label{sub:add_eval_results}

Other detailed performance data were shown in Table \ref{tab:convergence speed table}.
The scale experiment shown in Table \ref{tab:scale compare table}.
We also recorded the consensus distance in Tab \ref{tab:consensus}.

\begin{table*}[ht]
    \scriptsize
    \centering
    \caption{Accuracy Comparision of DPSGD and WPM with Devices $\in \{10, 15, 20\}$}
    \begin{tabular}{|p{\tableCellWidth}|p{\tableCellWidth}|p{\tableCellWidth}|p{\tableCellWidth}|p{\tableCellWidth}|p{\tableCellWidth}|p{\tableCellWidth}|}
     \hline \makecell[c]{\multirow{2}{*}{Devices}} & \multicolumn{3}{|c|}{i.i.d} & \multicolumn{3}{|c|}{Non-i.i.d} \\

     \cline{2-4} \cline{5-7} & \makecell[c]{500\\iterations} & \makecell[c]{Max} & \makecell[c]{Average} & \makecell[c]{500\\iterations} & \makecell[c]{Max} & \makecell[c]{Average} \\
     
     \hline \multicolumn{7}{|c|}{WPM(P=1)} \\
 
     \hline \makecell[c]{10} & \makecell[c]{90.15\%} & \makecell[c]{90.15\%} & \makecell[c]{86.80\%} & \makecell[c]{89.58\%} & \makecell[c]{89.77\%} & \makecell[c]{83.56\%}\\
     
     \makecell[c]{15} & \makecell[c]{90.17\%\\(0.02\%)} & \makecell[c]{90.17\%\\(0.02\%)} & \makecell[c]{86.91\%\\(0.13\%)} & \makecell[c]{88.54\%\\(-1.16\%)} & \makecell[c]{88.8\%\\(-1.08\%)} & \makecell[c]{82.41\%\\(-1.38\%)}\\
     
     \makecell[c]{20} & \makecell[c]{90.1\%\\(-0.06\%)} & \makecell[c]{90.1\%\\(-0.06\%)} & \makecell[c]{86.85\%\\(0.06\%)} & \makecell[c]{89.94\%\\(0.4\%)} & \makecell[c]{90.03\%\\(0.29\%)} & \makecell[c]{85.41\%\\(2.21\%)}\\

    \hline \multicolumn{7}{|c|}{WPM(P=15)} \\
    
    \hline \makecell[c]{10} & \makecell[c]{90.92\%} & \makecell[c]{90.95\%} & \makecell[c]{88.82\%} & \makecell[c]{89.28\%} & \makecell[c]{89.3\%} & \makecell[c]{86.84\%}\\
     \hline \makecell[c]{15} & \makecell[c]{90.93\%\\(0.01\%)} & \makecell[c]{90.95\%\\(0\%)} & \makecell[c]{88.74\%\\(-0.08\%)} & \makecell[c]{88.84\%\\(-0.49\%)} & \makecell[c]{88.91\%\\(-0.44\%)} & \makecell[c]{86.89\%\\(0.06\%)}\\
     \hline \makecell[c]{20} & \makecell[c]{90.93\%\\(0.01\%)} & \makecell[c]{90.93\%\\(-0.02\%)} & \makecell[c]{88.75\%\\(-0.08\%)} & \makecell[c]{89.4\%\\(0.13\%)} & \makecell[c]{89.4\%\\(0.11\%)} & \makecell[c]{87.43\%\\(0.68\%)}\\

    \hline
    \end{tabular}
    \label{tab:scaleWPM}
    \vspace{-1em}
\end{table*}

\begin{table*}[!bhp]
    \scriptsize
    \centering
    \caption{Convergence Speed Comparison with Different P}
    \begin{tabular}{|p{\conSpeedCellWidth}|p{\conSpeedCellWidth}|p{\conSpeedCellWidth}|p{\conSpeedCellWidth}|p{\conSpeedCellWidth}|p{\conSpeedCellWidth}|}
         
        \hline  \makecell[c]{Dataset} & \makecell[c]{WPM(P=1)} & \makecell[c]{WPM(P=3)} & \makecell[c]{WPM(P=5)} & \makecell[c]{WPM(P=9)} & \makecell[c]{WPM(P=15)} \\
        
        \hline \makecell[c]{i.i.d} & \makecell[c]{189} & \makecell[c]{189 \\(0\%)} & \makecell[c]{174\\(7.94\%)} & \makecell[c]{154\\(18.52\%)} & \makecell[c]{143\\(24.34\%)} \\
        
        \hline \makecell[c]{Non-i.i.d} & \makecell[c]{345} & \makecell[c]{332\\(3.77\%)} & \makecell[c]{224\\(35.07\%)} & \makecell[c]{213\\(38.26\%)} & \textbf{\makecell[c]{131\\(62.03\%)}} \\
         
         \hline
         
    \end{tabular}
    \label{tab:convergence speed table}
    \vspace{-1em}
\end{table*}

\begin{table*}[ht]
    \scriptsize
    \centering
    \caption{Convergence Speed Comparison with Scalability}
    \begin{tabular}{|p{\tableCellWidth}|p{\tableCellWidth}|p{\tableCellWidth}|p{\tableCellWidth}|p{\tableCellWidth}|p{\tableCellWidth}|p{\tableCellWidth}|}
     \hline \multirow{2}{*}{Methods} & \multicolumn{3}{|c|}{i.i.d} & \multicolumn{3}{|c|}{Non-i.i.d} \\
      \cline{2-4} \cline{5-7} & \makecell[c]{10} & \makecell[c]{15} & \makecell[c]{20} & \makecell[c]{10} & \makecell[c]{15} & \makecell[c]{20}  \\
    \hline WPM(P=1) & \makecell[c]{189} & \makecell[c]{189\\(0\%)} & \makecell[c]{190\\(-0.53\%)} & \makecell[c]{158} & \makecell[c]{158\\(0\%)} & \makecell[c]{142\\(10.13\%)}\\
    
    \hline WPM(P=15) & \makecell[c]{284} &  \makecell[c]{311\\(-9.51\%)} & \makecell[c]{269\\(5.28\%)} & \makecell[c]{131} &  \makecell[c]{127\\(3.05\%)} & \makecell[c]{113\\(13.74\%)}\\
    
    \hline
    \end{tabular}
    \label{tab:scale compare table}
    \vspace{-1em}
\end{table*}

\begin{table*}[!bhp]
    \scriptsize
    \centering
    \caption{Analysis of Consensus Distance  with Different P}
    
    \begin{tabular}{|p{\disTableCellWidth}|p{\disTableCellWidth}|p{\disTableCellWidth}|p{\disTableCellWidth}|p{\disTableCellWidth}|p{\disTableCellWidth}|p{\disTableCellWidth}|p{\disTableCellWidth}|p{\disTableCellWidth}|p{\disTableCellWidth}|p{\disTableCellWidth}|}
         
     \hline 
     
     \makecell[c]{\multirow{2}{*}{P Value}} & \multicolumn{5}{|c|}{i.i.d} & \multicolumn{5}{|c|}{Non-i.i.d} \\
     \cline{2-6} \cline{7-11} & \makecell[c]{100} & \makecell[c]{200} & \makecell[c]{300} & \makecell[c]{400} & \makecell[c]{500}
     & \makecell[c]{100} & \makecell[c]{200} & \makecell[c]{300} & \makecell[c]{400} & \makecell[c]{500} \\
     
     \hline \makecell[c]{3} & \makecell[c]{0.084} & \makecell[c]{0.074} & \makecell[c]{0.066} & \makecell[c]{0.058} & \makecell[c]{0.054}
     & \makecell[c]{0.179} & \makecell[c]{0.157} & \makecell[c]{0.140} & \makecell[c]{0.058} & \makecell[c]{0.054} \\
     
     \hline \makecell[c]{5} & \makecell[c]{0.311} & \makecell[c]{0.243} & \makecell[c]{0.201} & \makecell[c]{0.198} & \makecell[c]{0.199}
     & \makecell[c]{0.311} & \makecell[c]{0.243} & \makecell[c]{0.201} & \makecell[c]{0.123} & \makecell[c]{0.121} \\
     
     \hline \makecell[c]{9} & \makecell[c]{0.746} & \makecell[c]{0.601} & \makecell[c]{0.530} & \makecell[c]{0.519} & \makecell[c]{0.453}
     & \makecell[c]{0.225} & \makecell[c]{0.177} & \makecell[c]{0.155} & \makecell[c]{0.143} & \makecell[c]{0.139} \\
     
     \hline \makecell[c]{15} &\makecell[c]{1.250} & \makecell[c]{0.992} & \makecell[c]{0.969} & \makecell[c]{0.889} & \makecell[c]{0.877}
     & \makecell[c]{0.624} & \makecell[c]{0.433} & \makecell[c]{0.391} & \makecell[c]{0.282} & \makecell[c]{0.304} \\
         
    \hline
    \end{tabular}
    \label{tab:consensus}
\end{table*}

\end{document}